\def\year{2021}\relax
\documentclass[letterpaper]{article} 
\usepackage{aaai21}  
\usepackage{times}  
\usepackage{helvet} 
\usepackage{courier}  
\usepackage[hyphens]{url}  
\usepackage{graphicx} 
\usepackage{marvosym}
\usepackage{mathtools}
\usepackage{amsmath,amssymb}
\usepackage{algorithm} 
\usepackage{algorithmic}
\usepackage{bm}
\usepackage{array}
\usepackage{makecell}
\usepackage{amsmath}
\usepackage{pifont}
\newcommand{\cmark}{\ding{51}}%
\newcommand{\xmark}{\ding{55}}%
\newtheorem{theorem}{Theorem}
\newtheorem{proof}{Proof}
\newtheorem{lemma}{Lemma}

\usepackage{natbib}  
\usepackage{caption} 
\title{PDO-e{$\text{S}^\text{2}$}CNNs: Partial Differential Operator Based Equivariant Spherical CNNs}
\author {
    Zhengyang Shen\textsuperscript{\rm 1},
    Tiancheng Shen\textsuperscript{\rm 2,3},
    Zhouchen Lin \textsuperscript{\rm 4,5}\thanks{Corresponding authors},
    Jinwen Ma \textsuperscript{\rm 1\text{*}}\\
}
\affiliations {
    \textsuperscript{\rm 1} School of Mathematical Sciences and LMAM, Peking University \\
    \textsuperscript{\rm 2} Center for Data Science, Peking University \\
    \textsuperscript{\rm 3} The Chinese University of Hong Kong \\    
    \textsuperscript{\rm 4} Key Lab. of Machine Perception (MoE), School of EECS, Peking University \\
    \textsuperscript{\rm 5} Pazhou Lab, Guangzhou, China \\
    \{shenzhy, zlin\}@pku.edu.cn, tcshen@cse.cuhk.edu.hk, jwma@math.pku.edu.cn
}
\begin{document}

\maketitle
\begin{abstract}
Spherical signals exist in many applications, e.g., planetary data, LiDAR scans and digitalization of 3D objects, calling for models that can process spherical data effectively. It does not perform well when simply projecting spherical data into the 2D plane and then using planar convolution neural networks (CNNs), because of the distortion from projection and ineffective translation equivariance. 

Actually, good principles of designing spherical CNNs are avoiding distortions and converting the shift equivariance property in planar CNNs to rotation equivariance in the spherical domain. In this work, we use partial differential operators (PDOs) to design a spherical equivariant CNN, PDO-e{$\text{S}^\text{2}$}CNN, which is exactly rotation equivariant in the continuous domain. We then discretize PDO-e{$\text{S}^\text{2}$}CNNs, and analyze the equivariance error resulted from discretization. This is the first time that the equivariance error is theoretically analyzed in the spherical domain.  In experiments, PDO-e{$\text{S}^\text{2}$}CNNs show greater parameter efficiency and outperform other spherical CNNs significantly on several tasks.
\end{abstract}

\begin{table*}[t]
	\centering
	\small
		\begin{tabular}{l|l|l|l}
			\hline
			$\mathcal{S}^2$   & The sphere  &$SO(2),SO(3)$  &  Rotation groups \\
			$\alpha,\beta,\gamma$ & The ZYZ-Euler angles & $Z(\alpha),Y(\beta)$ & The rotations around $z$ and $y$ axes \\
			$R $ & $R\in SO(3)$ and $R=Z(\alpha_R)Y(\beta_R)Z(\gamma_R)$&$n=(0,0,1)^T$ & The north pole  \\
			$P$ & $P\in \mathcal{S}^2$, and $P(\alpha,\beta)=Z(\alpha)Y(\beta)n$ & $\bar P$ & The coset representative associated with $P$\\
			$\bar{P}\cdot SO(2)$ & $\{\bar{P}Z(\gamma)|\gamma\in[0,2\pi)\}$, the left coset of $SO(2)$ & $E\simeq F$ & $E$ is homeomorphic to $F$ \\
			$A_R$ & The 2D rotation matrix simplifying $Z(\gamma_R)$ & $C^{\infty}(\mathcal{S}^2)$ & The space of smooth function on $\mathcal{S}^2$ \\
			$C^{\infty}(SO(3))$&  The space of smooth function on $SO(3)$ & $s,so$ & $s\in C^{\infty}(\mathcal{S}^2)$ and $so\in C^{\infty}(SO(3))$ \\
			$\pi^{S}_{\widetilde{R}}[s], \pi^{SO}_{\widetilde{R}}[so]$& The group actions of $\widetilde R$ on $s$ and $so$ & $U_P$ & An open set of $\mathcal{S}^2$ containing $P$ \\
			$\widetilde U_p$ & $\widetilde U_p = \varphi_P(U_P)\subset \mathbb{R}^2$ & $\varphi_P$ & The homeomorphism from $U_P$ to $\widetilde U_P$\\
			$\bar s$ & The smooth function on $\mathbb{R}^3$ extended by $s$ & $H(\cdot,\cdot;\bm{w})$ & The polynomail parameterized by $\bm{w}$ \\
			$\partial/\partial x_i^{(A)}$ & The PDOs rotated by $A\in SO(2)$ & $\nabla_x [f],\nabla^2_x [f]$ & The gradient and the Hessian matrix of $f$\\
			$\nabla^{(A)}_x,(\nabla_x^{(A)})^2$& The operators defined in (\ref{gradient}) and (\ref{nabla2}) & $\chi^{(A)}$ &The differential operators defined in (\ref{aa}) \\
			$\Psi,\Phi$ & The mappings defined in (\ref{psi}) and (\ref{phi2}) & $\bm{I},\bm{F}$ & Discrete inputs and intermediate feature maps \\
			$f_P$ & $f_P=\bar s \cdot \varphi_P^{-1}$ & $O(\cdot)$ & The infinitesimal of the same order \\
			$D_P, \hat D_P$ & Partial derivatives matrix and its estimation & $\hat \nabla_x [f],\hat\nabla^2_x [f]$ & The estimations of $\nabla_x [f]$ and $\nabla^2_x [f]$ \\
			$\widetilde\chi^{(A)},\widetilde\Psi,\widetilde\Phi$& The discretizations of $\chi^{(A)},\Psi$ and $\Phi$ & $C_N$ & The $N$-ary cyclic group  \\	
			\hline
		\end{tabular}
	\caption{Summary of notations in this paper.}
	\label{notations}
\end{table*}

\section{Introduction}
Nowadays, many machine learning problems in computer vision require to process spherical data found in various applications; for instance, omnidirectional RGB-D images such as Matterport \cite{chang2017matterport3d:},  3D LiDAR scans from self-driving cars \cite{dewan2016motion-based} and molecular modelling \cite{boomsma2017spherical}. Unfortunately, naively mapping spherical signals to $\mathbb{R}^2$ and then using planar convolution neural networks (CNNs) is destined to fail, because this projection will result in space-varying distortions, and make shift equivariance ineffective.

Actually, the success of planar CNNs is mainly attributed to their shift equivariance \cite{cohen2016group}: shifting an image and then feeding it through multiple layers is the same as feeding the original image and then shifting the resulted feature maps. Since there do not exist translation symmetries in the spherical domain, a good principle of modifying planar CNNs to spherical CNNs is to convert the shift equivariance property to 3D rotation equivariance in the spherical domain. Motivated by this, \cite{cohen2018spherical} and \cite{esteves2018learning} propose spherical CNNs that are rotation equivariant. However, these methods represent the sphere using the spherical coordinates, which over-sample near the poles and cause significant distortion.

To avoid the impact of distortion, many recent works process spherical data using much more uniform representations. Among these methods,  \cite{cohen2019gauge} and \cite{zhang2019orientation-aware} approximate the sphere using the icosahedron and propose Icosahedral CNN and orientation-aware CNN, respectively. Specifically, Icosahedral CNN \cite{cohen2019gauge} is rotation equivariant while orientation-aware CNN \cite{zhang2019orientation-aware} is beneficial for some orientation-aware tasks, such as semantic segmentation with preferred orientation. However, these methods need project spherical data to the icosahedron, resulting in inaccurate representations.

Actually, there exist some discretizations of the sphere that are both uniform and accurate, like the icosahedral spherical mesh \cite{baumgardner1985icosahedral} and the HealPIX \cite{gorski2005healpix}. However, these representations are non-Euclidean structured grids \cite{bronstein2017geometric}, which have no uniform locality, thus conventional convolutions defined in the Euclidean case (e.g., square lattices) cannot work on them. Accordingly, \cite{jiang2019spherical} propose MeshConvs, which use orientable parameterized partial differential operators (PDOs) to process spherical signals represented by non-Euclidean structured grids. However, MeshConvs are not rotation equivariant.

In order to address the above problems, we combine the advantages of \cite{cohen2019gauge} and \cite{jiang2019spherical} together, and propose PDO-e{$\text{S}^\text{2}$}CNN, which is an orientable rotation equivariant spherical CNN based on PDOs. The distinction from \cite{cohen2019gauge} is that our model is orientation-aware and can work on much more accurate non-Euclidean structured representations instead of icosahedron, and the difference from \cite{jiang2019spherical} is that ours is rotation equivariant.

Our contributions are as follows:
\begin{itemize}
	\item We use PDOs to design an orientable spherical CNN that is exactly rotation equivariant in the continuous domain.
	
	\item The equivariance of the PDO-e{$\text{S}^\text{2}$}CNN becomes approximate after the discretization, and it is the first time that the theoretical equivariance error analysis is provided when the equivariance is approximate in the spherical domain.
	
	\item PDO-e{$\text{S}^\text{2}$}CNNs show greater parameter efficiency and perform very competitively on spherical MNIST classification, 2D-3D-S image segmentation and QM7 atomization energy prediction tasks.
\end{itemize} 

The paper is organized as follows. In Related Work, we review some works related to spherical CNNs. In Prior Knowledge, we introduce some prior knowledge to make our work easy to understand. In PDO-e{$\text{S}^\text{2}$}CNN, we use orientable parameterized PDOs to design PDO-e{$\text{S}^\text{2}$}CNN, which is exactly equivariant over $SO(3)$ in the continuous domain. In Implementation, we use Taylor's expansion to estimate PDOs accurately, implement PDO-e{$\text{S}^\text{2}$}CNN in the discrete domain, and provide the equivariance error analysis. In Experiments, we evaluate our method on multiple tasks.

\section{Related Work \label{section2}}
The most straightforward method to process spherical signals is mapping them into the planar domain via the equirectangular projection \cite{su2017learning}, and then using 2D CNNs. However, this projection will result in severe distortion. \cite{coors2018spherenet:} and \cite{zhao2018distortion-aware} implement CNNs on the tangent plane of the spherical image to reduce distortions. Even though, such methods are not equivariant in the spherical domain.

Actually, many works \cite{cohen2016group,cesa2019general,shen2020pdo,sosnovik2019scale,weiler20183d,ravanbakhsh2017equivariance} focus on incorporating equivariance into networks. For spherical data, some works \cite{bruna2014spectral,frossard2017graph-based,perraudin2019deepsphere:,defferrard2020deepsphere:} represent the sampled sphere as a graph connecting pixels according to distance between them and utilize graph-based methods to process it. \cite{perraudin2019deepsphere:} propose DeepSphere using isotropic filters, and achieve rotation equivariance. \cite{defferrard2020deepsphere:} improve DeepSphere and achieve a controllable tradeoff between cost and equivariance. However, the isotropic filters they use significantly restrict the capacity of models.

Also, there exist some works \cite{cohen2018spherical,esteves2018learning,kondor2018clebsch-gordan} using anisotropic filters to achieve rotation equivariance. Specifically, \cite{cohen2018spherical} extend the group equivariance theory into the spherical domain and use a generalized Fourier transform for implementation. However, these methods only work on nonuniform grids which over-sample near the poles. \cite{cohen2019gauge} further extend group equivariance to gauge equivariance, which is automatically $SO(3)$ equivariant in the spherical domain. However, their theory cannot show how the feature maps transform w.r.t. rotation transformations explicitly whereas ours can, which makes our theory more transparent and explainable.
\cite{cohen2019gauge} implement gauge equivariant CNNs on the surface of the icosahedron. The icosahedron is not an accurate discretization of the sphere, so their equivariance is weak. By contrast, our method can be applied on accurate discretizations of the sphere, achieving much better equivariance consequently.

Particularly, empirical results \cite{jiang2019spherical,zhang2019orientation-aware} show that orientation-aware CNNs can be beneficial for some tasks with orientation information. \cite{zhang2019orientation-aware} use north-aligned filters to achieve orientation-awareness, while \cite{jiang2019spherical} use orientable PDOs. In addition, \cite{jiang2019spherical} can process spherical signals on non-Euclidean structured grids easily using PDOs. However, their models are not rotation equivariant. Our PDO-e{$\text{S}^\text{2}$}CNN furthermore incorporates equivariance into the model, and introduces a new weight sharing scheme across filters, which brings greater parameter efficiency.

\section{Prior Knowledge\label{section3}}  

\subsection{Parameterization of $\mathcal{S}^2$ and $SO(3)$} 
We use $\mathcal{S}^2$ and $SO(3)$ to denote a sphere and a group of 3D rotations, respectively. Formally,
\begin{align*}
	&\mathcal{S}^2 =\{(x_1,x_2,x_3)|\|x\|_2=1\},\\
	&SO(3) = \{R\in \mathbb{R}^3|R^TR=I,\det(R)=1\}.
\end{align*}
We use the ZYZ Euler parameterization for $SO(3)$. An element $R\in SO(3)$ can be written as
\begin{equation*}
	R=Z(\alpha_R)Y(\beta_R)Z(\gamma_R),
\end{equation*}
where ZYZ-Euler angles $\alpha_R \in [0,2\pi),\beta_R \in [0,\pi]$ and $\gamma_R \in [0,2\pi)$, and $Z(\alpha)$ and $Y(\beta)$ are rotations around $z$ and $y$ axes, respectively. To be specific,
\begin{scriptsize}
	\begin{align*}
		Z(\alpha)=\left[
		\begin{array}{ccc}
			\cos\alpha & -\sin\alpha & 0\\
			\sin\alpha & \cos\alpha & 0\\
			0 & 0 & 1\\
		\end{array}
		\right],Y(\beta)=\left[
		\begin{array}{ccc}
			\cos\beta & 0  & \sin\beta\\
			0  & 1  & 0\\
			-\sin\beta & 0 & \cos\beta\\
		\end{array}
		\right].
	\end{align*}
\end{scriptsize}
Accordingly, we have a related parameterization for the sphere. An element $P\in \mathcal{S}^2$ can be written as $P(\alpha,\beta)=Z(\alpha)Y(\beta)n$, where $n$ is the north pole, i.e., $n=(0,0,1)^T$. Conversely, we can also calculate $\alpha$ and $\beta$ if $P=(x_1,x_2,x_3)^T$ is given. To be specific, if $P=(0,0,1)^T$, we take $\alpha=\beta=0$; if $P=(0,0,-1)^T$, we take $\alpha=0$ and $\beta=\pi$; otherwise, we have
\begin{small}
	\begin{align*}
		\alpha = 
		\begin{cases}
			\arccos \left(\frac{x_1}{\sqrt{x_1^2+x_2^2}}\right) & \text{$x_2\geq 0$}\\
			2\pi-\arccos \left(\frac{x_1}{\sqrt{x_1^2+x_2^2}}\right) &\text{$x_2< 0$}
		\end{cases},\,\,\beta =\arccos(x_3).
	\end{align*}
\end{small}
\begin{figure}
	\centering
	\includegraphics[scale=0.25]{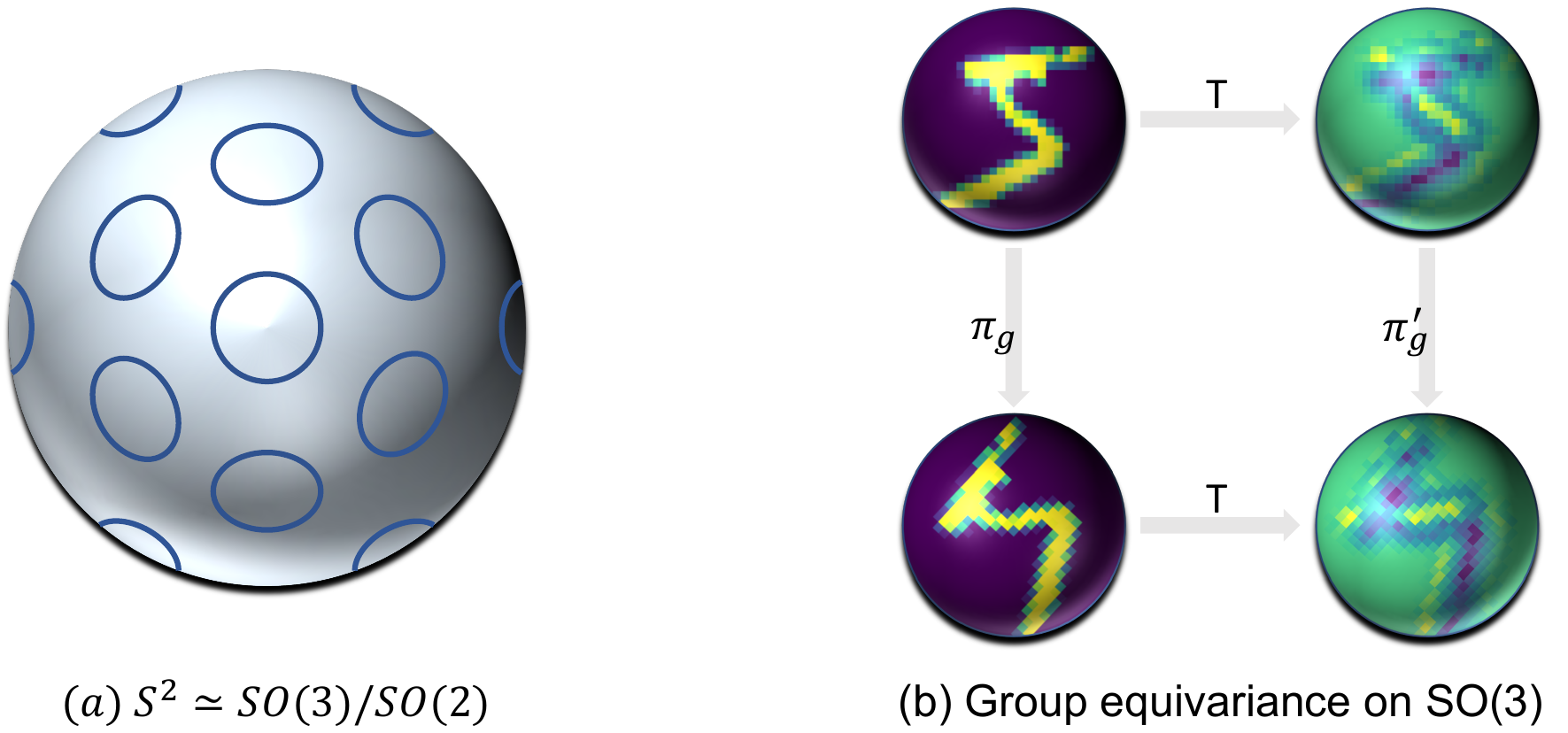}
	\caption{(a) $\mathcal{S}^2\simeq SO(3)/SO(2)$. $SO(3)$ can be viewed as a bundle of circles over the sphere; (b) Group equivariance on $SO(3)$. Transforming an input by a transformation $g\in SO(3)$ and then passing it through the mapping $T$ is equivalent to first mapping it through $T$ and then transforming the representation. }
	\label{figure1}
\end{figure}

This parameterization makes explicit the fact that the sphere is a quotient $\mathcal{S}^2\simeq SO(3)/SO(2)$\footnote{Given a group $\mathcal{G}$ and its subgroup $\mathcal{H}$, the left cosets $g\mathcal{H}$ of $\mathcal{H}$ partition $\mathcal{G}$, where $g\in \mathcal{G}$. We denote the set of left cosets as $\mathcal{G}/\mathcal{H}$. $E\simeq F$ denotes that $E$ is homeomorphic to $F$.}, where $SO(2)$ is the subgroup of $SO(3)$ and contains the rotations around the $z$ axis. Elements of the subgroup $SO(2)$ leave the north pole invariant, and have the form $Z(\gamma)$. The point $P(\alpha,\beta)\in \mathcal{S}^2$ is associated with the coset representative $\bar{P}=Z(\alpha)Y(\beta)\in SO(3)$. This element represents the left coset $\bar{P}\cdot SO(2)=\{\bar{P}Z(\gamma)|\gamma\in[0,2\pi)\}$. Intuitively, $SO(3)$ can be viewed as a bundle of circles ($SO(2)$) over the sphere, as we show in Figure \ref{figure1}(a). In this way, $\forall R\in SO(3)$, $R\in \bar{P}_RSO(2)$, where $\bar{P}_R=Z(\alpha_R)Y(\beta_R)$. As a result, we can parameterize $R$ as $(P_R,A_R)$, where $P_R=\bar{P}_Rn\in \mathcal{S}^2$ and $A_R\in SO(2)$. Specifically, $A_R$ is a 2D rotation matrix, which is a simplification of $Z(\gamma_R)$, i.e.,
\begin{equation*}
	A_R=\left[
	\begin{array}{p{1.1cm}<{\centering} p{1.1cm}<{\centering}}
		$\cos\gamma_R$ & $-\sin\gamma_R$ \\
		$\sin\gamma_R$ & $\cos\gamma_R$ \\
	\end{array}
	\right].
\end{equation*}

\subsection{Group Actions on Spherical Functions}
Inputs and feature maps can be naturally modeled as functions in the continuous domain. Specifically, we model the input $s$ as a smooth function on $\mathcal{S}^2$ and the intermediate feature map $so$ as a smooth function on $SO(3)$. Particularly, the smoothness of $so$ means that if we use the parameterization of $SO(3)$ mentioned above, the feature map $so(P,A)$ is smooth w.r.t. $P$ when $A$ is fixed. So $so$ can also be viewed as a smooth spherical function with infinite channels indexed by $A\in SO(2)$. We use $C^{\infty}(\mathcal{S}^2)$ and $C^{\infty}(SO(3))$ to denote the function spaces of $s$ and $so$, respectively .

In this way, rotation transformations acting on inputs and feature maps can be mathematically formulated as follows. \\
\textbf{Actions on Inputs}\quad Suppose that $s\in C^\infty(\mathcal{S}^2)$ and $\widetilde{R} \in SO(3)$, then $\widetilde{R}$ acts on $s$ in the following way: 
\begin{align*}
	\forall P \in \mathcal{S}^2,\quad \pi^{S}_{\widetilde{R}}[s](P)=s\left({\widetilde{R}}^{-1}P\right).
\end{align*}
\textbf{Actions on Feature Maps}\quad Suppose that $so \in C^\infty(SO(3))$ and $\widetilde{R} \in SO(3)$, then $\widetilde{R}$ acts on $so$ in the following way: 	
\begin{align}
	\forall R \in SO(3),\quad \pi^{SO}_{\widetilde{R}}[so](R)=so\left({\widetilde{R}}^{-1}R\right).
	\label{31}
\end{align}
If we use the parameterization of $SO(3)$, (\ref{31}) is of the following more intuitive form:
\begin{align*}
	\pi^{SO}_{\widetilde{R}}[so](P_R,A_R)&=so\left(P_{\widetilde{R}^{-1}R},A_{\widetilde{R}^{-1}R}\right)\\
	&=so\left(\widetilde{R}^{-1}P_R,A_{\widetilde{R}^{-1}R}\right),
\end{align*} 
where $(P_R,A_R)$ is the representation of $R$ and $P_{\widetilde{R}^{-1}R}=\widetilde{R}^{-1}Rn=\widetilde{R}^{-1}P_R$.

\subsection{Group Equivariance}
Equivariance measures how the outputs of a mapping transform in a predictable way with the transformation of the inputs. To be specific, let $T$ be a mapping, which could be represented by a deep neural network from the input feature space to the output feature space, and $\mathcal{G}$ is a transformation group. $T$ is called group equivariant if it satisfies
\begin{align*}
	\forall g \in \mathcal{G},\quad T[\pi_g[f]]=\pi^{\prime}_g[T[f]],
\end{align*}
where $f$ can be any input feature map in the input feature space, and $\pi _g$ and $\pi^{\prime}_g$ denote how the transformation $g$ acts on input features and output features, respectively. 

In our theory, we take the group $\mathcal{G}$ as $SO(3)$, and then focus on utilizing PDOs to design a neural network equivariant to $SO(3)$, as shown in Figure \ref{figure1}(b).

\begin{figure}
	\centering
	\includegraphics[scale=0.3]{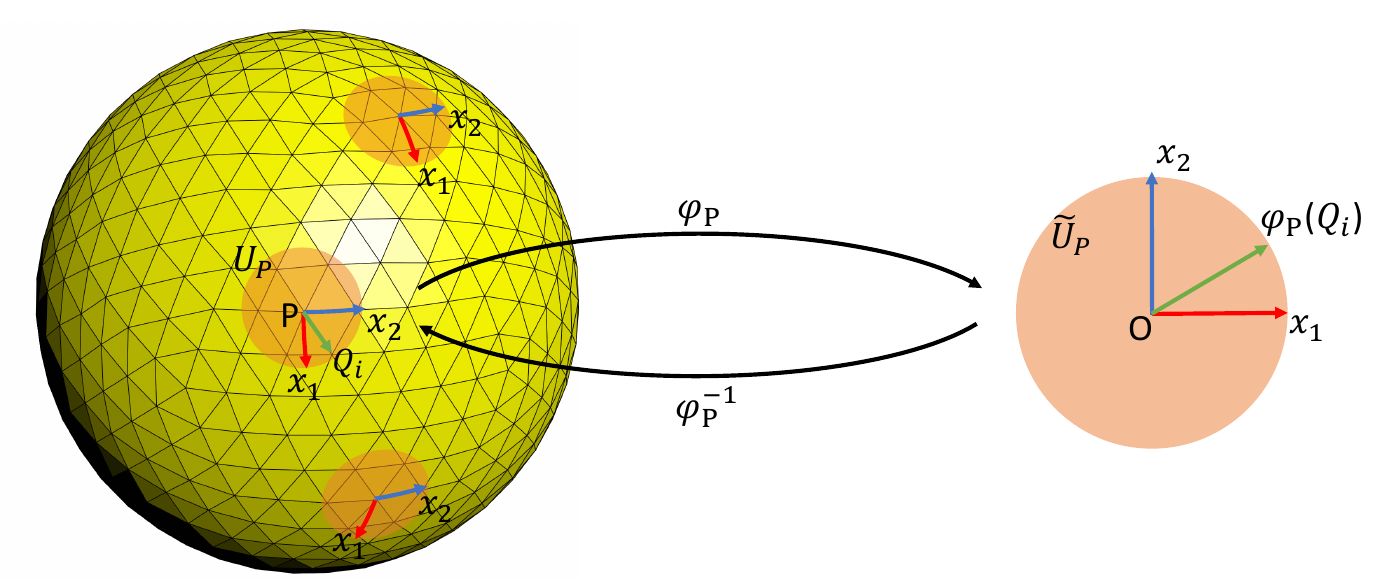}
	\caption{For any $P\in \mathcal{S}^2$, a homeomorphism $\varphi_P$ maps the chart $U_P\subset \mathcal{S}^2$ to an open subset $\widetilde U_P\subset \mathbb{R}^2$. The sphere is presented by a level-$3$ icosahedral mesh. }
	\label{figure2}
\end{figure}

\section{PDO-e{$\text{S}^\text{2}$}CNNs \label{section4}}
\subsection{Chart-based PDOs }
We define an atlas to help define PDOs acting on the spherical functions uniformly. To be specific, an atlas for $\mathcal{S}^2$ is a collection of charts whose domains cover $\mathcal{S}^2$. We denote the atlas as $\{(U_P,\varphi_P)|P\in \mathcal{S}^2\}$, where $U_P$ is an open subset of $\mathcal{S}^2$ containing $P$ and $\varphi_P:U_P\rightarrow \widetilde{U}_P$ is a homeomorphism from the chart $U_P$ to an open subset $\widetilde{U}_P=\varphi_P(U_P)\subset \mathbb{R}^2$ and $\varphi_P(P)=0$. The form of $\varphi_P$ is given by 
\begin{equation}
	\varphi_P^{-1}(x_1,x_2)=\bar{P}
	\left(x_1,x_2,\sqrt{1-|x|^2}\right)^T.
	\label{phi}
\end{equation}
In this way, as shown in Figure \ref{figure2}, for any point $P\in \mathcal{S}^2$ (except poles), $x_1$ resp. $x_2$ point to the north-south and east-west directions in the chart $U_P$, and the homeomorphism $\varphi_P$'s are uniformly defined over the sphere, which relate to orientable and uniform PDOs over the sphere.

In order to use PDOs, we suppose that the spherical function $s$ is smooth and denote it as $s\in C^{\infty}(\mathcal{S}^2)$. $s$ can always be extended to a smooth function $\bar{s}$ defined on $\mathbb{R}^3$, and we denote it as $\bar{s}\in C^{\infty}(\mathbb{R}^3)$. We emphasize that we need not obtain $\bar{s}$ explicitly from the given $s$, whereas we only use this notation for ease of derivation. Then the PDOs $\partial/\partial x_i$ and $\partial^2/\partial x_i\partial x_j(i,j=1,2)$\footnote{We only consider the PDOs up to the second order in this work.} act on the spherical function $s$ in the way that these PDOs act on the composite function $\bar{s}\cdot \varphi_P^{-1}\in C^{\infty}(\mathbb{R}^2)$\footnote{We use $[\cdot]$ to denote that an operator acts on a function.}. Formally, $\forall P\in \mathcal{S}^2$,
\begin{align*}
	\frac{\partial}{\partial x_i}[s](P)&=\frac{\partial}{\partial x_i}\left[\bar{s}\cdot\varphi_P^{-1}\right](0),\\
	\frac{\partial^2}{\partial x_i\partial x_j}[s](P)&=\frac{\partial^2}{\partial x_i\partial x_j}\left[\bar{s}\cdot\varphi_P^{-1}\right](0).
\end{align*}

By contrast, \cite{jiang2019spherical} define PDOs based on the spherical coordinates, which have high resolution near the pole and low resolution near the equator. So the scales of their PDOs are dependent on the latitudes. By contrast, the scales of our chart-based PDOs are independent of locations, resulting in much more uniform feature extration. Our definition of PDOs is also different from that in conventional manifold calculus in that we can deal with second-order PDOs without defining a smooth vector field. Actually, it is impossible to define a non-trivial smooth vector field over the sphere due to the hairy ball theorem \cite{milnor1978analytic}.

\subsection{Rotated Parameterized Differential Operators \label{rotated}}
Following \cite{jiang2019spherical,ruthotto2018deep,shen2020pdo}, we parameterize convolution kernels using a linear combination of PDOs. Specifically, we refer to $H$ as a parameterized second-order polynomial of $2$ variables, i.e.,
\begin{equation}
	H(u,v;\bm{w}) = w_1 + w_2u+ w_3v + w_4u^2+w_5uv+ w_6 v^2,\label{poly}
\end{equation}
where $\bm{w}$ are learnable parameters. If we take $u=\partial/\partial x_1$ and $v=\partial/\partial x_2$, then $H(\partial/\partial x_1,\partial/\partial x_2;\bm{w})$ becomes a linear combination of PDOs. For example, if $H(u,v;\bm{w})=u^2+uv$, then $H(\partial /\partial x_1,\partial/\partial x_2;\bm{w})=\partial^2/\partial x_1^2+\partial^2/\partial x_1\partial x_2$. 

We rotate these PDOs with a $2\times2$ rotation matrix $A\in SO(2)$, and obtain the following rotated parameterized differential operators:
\begin{align}
	\chi^{(A)}=H\left(\frac{\partial}{\partial x_1^{(A)}},\frac{\partial}{\partial x_2^{(A)}};\bm{w}\right),
	\label{aa}
\end{align}
where 
\begin{equation}
	\left(\frac{\partial}{\partial x_1^ {(A)}},\frac{\partial}{\partial x_2^ {(A)}}\right)^T= A^{-1} \left(\frac{\partial}{\partial x_1},\frac{\partial}{\partial x_2}\right)^T.
	\label{21}
\end{equation}
As a compact form, we can also rewrite (\ref{21}) as 
\begin{equation}
	\nabla_x^{(A)} = A^{-1}\nabla_x,\label{gradient}
\end{equation}
where $\nabla_x=(\partial/\partial x_1,\partial/\partial x_2)^T$ is the gradient operator. (\ref{21}) is equivalent to first rotating the coordinate system by $Z$, and then calculating gradients. In addition, it is easy to get that 
\begin{align}
	\left(\nabla_x^{(A)}\right)^2&\coloneqq
	\left[
	\begin{array}{cc}
		\frac{\partial}{\partial x_1^ {(A)}}\frac{\partial}{\partial x_1^ {(A)}} & \frac{\partial}{\partial x_1^ {(A)}}\frac{\partial}{\partial x_2^ {(A)}}
		\vspace{5pt}\\
		\frac{\partial}{\partial x_1^ {(A)}}\frac{\partial}{\partial x_2^ {(A)}} & \frac{\partial}{\partial x_2^ {(A)}}\frac{\partial}{\partial x_2^ {(A)}}
	\end{array}
	\right]\label{nabla2}\\
	&= A^{-1} \left[
	\begin{array}{cc}
		\frac{\partial^2}{\partial x_1^2} & \frac{\partial^2}{\partial x_1\partial x_2}\\
		\vspace{-8pt}\\
		\frac{\partial^2}{\partial x_1\partial x_2} & \frac{\partial^2}{\partial x_2^2}\\
	\end{array}
	\right]A=A^{-1}\nabla_x^2A. \notag
\end{align}
To make it more explicit, we emphasize that by the definition in (\ref{aa}), $\chi^{(A)}$'s are identical polynomials w.r.t. $\partial/\partial x_1^{(A)}$'s and $\partial/\partial x_2^{(A)}$'s, but different polynomials w.r.t. $\partial/\partial x_1$ and $\partial/\partial x_2$. To be specific,
\begin{scriptsize}
	\begin{align}
		\chi^{(A)}=&w_1+(w_2,w_3)\nabla_x^{(A)}+
		\bigg\langle\left[
		\begin{array}{cc}
			w_4 & \frac{w_5}{2}\\
			\frac{w_5}{2} & w_6\\
		\end{array}
		\right],
		\left(\nabla_x^{(A)}\right)^2
		\bigg\rangle\notag\\
		=&w_1+(w_2,w_3)A^{-1}\nabla_x+
		\bigg\langle\left[
		\begin{array}{cc}
			w_4 & \frac{w_5}{2}\\
			\frac{w_5}{2} & w_6\\
		\end{array}
		\right],
		A^{-1}\nabla_x^2A \bigg\rangle\notag\\
		=&w_1+(w_2,w_3)A^{-1}\nabla_x+
		\bigg\langle A\left[
		\begin{array}{cc}
			w_4 & \frac{w_5}{2}\\
			\frac{w_5}{2} & w_6\\
		\end{array}
		\right]A^{-1},
		\nabla_x^2\bigg\rangle,
		\label{coef}
	\end{align}
\end{scriptsize}
where $\langle\cdot,\cdot \rangle$ denotes the inner product. Particularly, these differential operators $\chi^{(A)}$'s share parameters $\bm{w}$, indicating great parameter efficiency.

From another point of view, the rotation of differential operators can also be viewed as changing the coefficients of PDOs (see (\ref{coef})), without changing the orientations of PDOs. Consequently, the rotated parameterized differential operators, $\chi^{(A)}$'s, and the subsequent PDO-e{$\text{S}^\text{2}$}CNN are still orientable. By contrast, some rotation equivariant spherical CNNs, such as Icosahedral CNNs \cite{cohen2019gauge}, assume no preferred orientation, so they are not orientable.

\subsection{Equivariant Differential Operators}
We define two mappings, $\Psi$ and $\Phi$, using the above-mentioned differential operators, $\chi^{(A)}$'s. To be specific, we use $\Psi$ to deal with inputs, which maps an input $s$ to a feature map defined on $SO(3)$: $\forall R\in SO(3)$,
\begin{align}
	\Psi [s](R) = \Psi [s](P_R,A_R)=\chi^{(A_R)}[s](P_R).
	\label{psi}
\end{align}

Then, we use $\Phi$ to deal with the resulting feature maps, which maps one feature map defined on $SO(3)$ to another feature map defined on $SO(3)$: $\forall R\in SO(3)$,
\begin{align}
	\Phi [so](R) &= \Phi [so](P_R,A_R)\notag\\
	&=\int_{SO(2)} \chi^{(A_R)}_{A}\,\,[so](P_R,A_RA) d\nu(A)\label{phi2},
\end{align}
where  $\nu$ is a measure on $SO(2)$. As for $\chi_A^{(A_R)}$, we use the subscript $A$ to distinguish the differential operators parameterized by different $\bm{w}_A$'s. The $so$ on the right hand side should be viewed as a spherical function indexed by $A_RA$ when the operator $\chi^{(A_R)}_A$ acts on it. 

Finally, we prove that the above two mappings, $\Psi$ and $\Phi$, are equivariant under arbitrary rotation transformation $\widetilde{R}\in SO(3)$ and show how the outputs transform w.r.t. the transformation of inputs. The proofs of theorems can be found in the Supplementary Material.

\begin{theorem}
	If $s \in C^{\infty}(\mathcal{S}^2)$ and $so \in C^{\infty}(SO(3))$, $\forall \widetilde{R}\in SO(3)$, we have
	\begin{align}
		\Psi \left[\pi^{S}_{\widetilde R}[s]\right]&=\pi^{SO}_{\widetilde R}\left[\Psi [s]\right],\label{equi1}\\	
		\Phi \left[\pi^{SO}_{\widetilde R}[so]\right] &= \pi^{SO}_{\widetilde R}\left[\Phi [so]\right].\label{4}
	\end{align}
	\label{theorem1}
\end{theorem}
\subsection{Equivariant Network Architectures}\label{general}
It is easy to use the above-mentioned two equivariant mappings, $\Psi$ and $\Phi$, to design an equivariant network. To be specific, according to the working spaces, we set a $\Psi$ as the first layer, followed by multiple $\Phi$'s, inserted by pointwise nonlinearities $\sigma(\cdot)$, e.g., ReLUs, which do not disturb the equivariance. Finally, we can get an equivariant network architecture $T[s]=\Phi^{(L)}\left[\cdots\sigma\left(\Phi^{(1)}\left[\sigma(\Psi[s])\right]\right)\right]$. \begin{theorem}
	If $ s \in C^{\infty}(\mathcal{S}^2)$, $\forall \widetilde{R}\in SO(3)$, we have
	\begin{align*}
		T\left[\pi^{S}_{\widetilde R}[s]\right]= \pi^{SO}_{\widetilde R}\left[T[s]\right].
	\end{align*}
	\label{theorem3}
\end{theorem}
That is, transforming an input $s$ by a transformation $\widetilde R$ (forming $\pi^{S}_{\widetilde R}$) and then passing it through the network $T$ gives the same result as first mapping $s$ through $T$ and then transforming the representation.

As discussed above, we only consider the case where inputs, $s$, and intermediate feature maps over $SO(3)$, $so$, only consist of single channel. In fact, our theory can be easily extended to a more general case where inputs and feature maps consist of multiple channels, and we only need to use multiple $\Psi$'s and $\Phi$'s to process inputs and generate outputs. 

Besides, in conventional CNNs, we always use $1\times 1$ convolutions to change the numbers of channels without introducing too many parameters. In PDO-e{$\text{S}^\text{2}$}CNN, this can be easily achieved by taking $\bm{w}$ as a one-hot vector. The details are given in the Supplementary Material. We can also incorporate equivariance into other architectures, e.g., ResNets, because shortcut connections do not disturb equivariance. 

\section{Implementation \label{section5}}
\subsection{Icosahedral Spherical Mesh}
In practice, spherical data are always given on discrete domain, instead of continuous domain. The icosahedral spherical mesh \cite{baumgardner1985icosahedral} is among the most uniform and accurate discretization of the sphere. Specifically, a spherical mesh can be obtained by progressively subdividing each face of the unit icosahedron into four triangles and reprojecting each node to unit distance from the origin. We start with the unit icosahedron as the level-0 mesh, and each progressive mesh resolution is one level above the previous. The level-3 icosahedral mesh is shown in Figure \ref{figure2}. The subdivision scheme for triangles also provides a natural coarsening and refinement scheme for the grid, which allows for easy implementations of downsampling and upsampling routines associated with CNN architectures. We emphasize that our method is not limited to the icosahedral spherical mesh, but can also use other discrete representations of the sphere easily, like the HealPIX \cite{gorski2005healpix}. In this work, we use the icosahedral spherical mesh for ease of implementation.
%
%
%

\subsection{Estimation of Partial Derivatives \label{espdo}}
We view the input spherical data $\bm{I}$ as a discrete function sampled from a smooth spherical function $s$ on the icosahedral spherical mesh vertices $\Omega\subset \mathcal{S}^2$, where $\bm{I}(P)=s(P),\forall P\in \Omega$, and use a numerical method to estimate partial derivatives at $P\in \Omega$ in the discrete domain. Firstly, we use $\varphi_P$ to map $P$ and $Q_i(i=1,2,\cdots,m)$ into an open set $\widetilde U_P\subset \mathbb{R}^2$, where $Q_i\in \Omega$ are the neighbor nodes of $P$ (see Figure \ref{figure2})\footnote{We only consider the neighbor nodes of $P$, in analogy with the commonly-used $3\times 3$ convolutions in planar CNNs.}. As a result, we get $\varphi_P(P)=0$, and $\varphi_P(Q_i)=(x_{i1},x_{i2})$, where $\forall i=1,2,\cdots,m$,
\begin{equation*}
	\left(x_{i1},x_{i2},\sqrt{1-x_{i1}^2-x_{i2}^2}\right)^T
	=\bar{P}^{-1}Q_i.
\end{equation*}

We denote $f_P=\bar s\cdot \varphi_P^{-1}$, so $f_P(0)=s(P)=\bm{I}(P)$ and $f_P(x_{i1},x_{i2})=s(Q_i)=\bm{I}(Q_i)$. We use Taylor's expansion to expand $f_P$ at the original point, then we have that $\forall i=1,2,\cdots,m$,
\begin{small}
	\begin{align}
		f_P(x_{i1},x_{i2})=&f_P(0,0) + x_{i1}\frac{\partial f_P}{\partial x_1}+x_{i2}\frac{\partial f_P}{\partial x_2}+\frac{1}{2}x_{i1}^2\frac{\partial^2 f_P}{\partial x_1^2}\notag\\
		&+x_{i1}x_{i2} \frac{\partial^2 f_P}{\partial x_1\partial x_2} +\frac{1}{2}x_{i2}^2\frac{\partial^2 f_P}{\partial x_2^2} + O(\rho_i^3)\label{fp}
	\end{align}
\end{small}
where all above partial derivatives are evaluated at $(0,0)$, and $\rho_i=\sqrt{x_{i1}^2+x_{i2}^2}$. Thus we have
\begin{footnotesize}
	\begin{equation*}
		\left[
		\begin{array}{p{2.7cm}<{\centering}}
			$\vdots$  \\
			$f_P(x_{i1},x_{i2})-f_P(0)$ \\
			$\vdots$\\
		\end{array}
		\right]
		\approx 
		\left[
		\begin{array}{p{0.2cm}<{\centering} p{0.2cm}<{\centering}p{0.2cm}<{\centering}p{0.55cm}<{\centering}p{0.3cm}<{\centering}}
			$\vdots$ & $\vdots$  & $\vdots$  & $\vdots$  & $\vdots$  \\
			$x_{i1}$ & $x_{i2}$ & $\frac{x_{i1}^2}{2}$ & $x_{i1}x_{i2}$ & $\frac{x_{i2}^2}{2}$\\
			$\vdots$  & $\vdots$  & $\vdots$ & $\vdots$  & $\vdots$  \\
		\end{array}
		\right]D_P,
		\label{approx}
	\end{equation*}
\end{footnotesize}
where $D_P$ is a partial derivatives matrix:
\begin{equation*}
	D_P=\left(\frac{\partial f_P}{\partial x_1},\frac{\partial f_P}{\partial x_2},\frac{\partial^2 f_P}{\partial x_1^2},\frac{\partial^2 f_P}{\partial x_1x_2},\frac{\partial^2 f_P}{\partial x_2^2}\right)^T\bigg|_{x_1=x_2=0}.
\end{equation*}
We denote the above approximate equations as $F_P\approx V_PD_P$, and use the least square method to estimate $D_P$:
\begin{equation*}
	{\hat D_P}=\mathop{\arg\min}_{D} \|V_PD-F_P\|_2= (V_P^TV_P)^{-1}V_P^TF_P.
\end{equation*}

Actually, we can easily estimate any partial derivatives using the similar method so long as we employ the appropriate Taylor's expansions. By contrast, \cite{jiang2019spherical} can only deal with limited PDOs, including $\partial/\partial x_1,\partial/\partial x_2$, and the Laplacian operator.

\subsection{Discretization of $SO(2)$}
As it is impossible to go through all the $A\in SO(2)$ in (\ref{psi}) and (\ref{phi2}), we need to discretize $SO(2)$. To be specific, we discretize the continuous group $SO(2)$ as the $N$-ary cyclic group $C_N$, where $C_N=\{e=A_0,A_1,\cdots,A_{N-1}\}$, and 

\begin{equation*}
	A_i=\left[
	\begin{array}{p{1.1cm}<{\centering} p{1.1cm}<{\centering}}
		$\cos\frac{2\pi i}{N}$ & $-\sin\frac{2\pi i}{N}$ \\
		$\sin\frac{2\pi i}{N} $ & $\cos\frac{2\pi i}{N}$ \\
	\end{array}
	\right].
\end{equation*}

Correspondingly, (\ref{psi}) should be discretized as: $\forall P\in \Omega$ and $i=0,1,\cdots,N-1$,
\begin{scriptsize}
	\begin{align*}
		&\widetilde \Psi [\bm{I}](P,i)=\widetilde \chi^{(A_i)}[\bm{I}](P)\\
		=&\left(w_1+(w_2,w_3)A_i^{-1}\widehat\nabla_x+\bigg\langle A_i\left[
		\begin{array}{cc}
			w_4 & \frac{w_5}{2}\\
			\frac{w_5}{2} & w_6\\
		\end{array}
		\right]A_i^{-1},
		\hat\nabla_x^2\bigg\rangle\right)\left[f_P\right](0)\notag\\
		=&w_1f_P(0)+(w_2,w_3)A_i^{-1}\hat\nabla_x\left[f_P\right](0)\\
		&+\bigg\langle A_i\left[
		\begin{array}{cc}
			w_4 & \frac{w_5}{2}\\
			\frac{w_5}{2} & w_6\\
		\end{array}
		\right]A_i^{-1},
		\hat\nabla_x^2\left[f_P\right](0)\bigg\rangle\notag,
	\end{align*}
\end{scriptsize}
where the partial derivatives are estimated using $\bm{I}$. In this way, when viewed as a spherical function, the output $\widetilde\Psi[\bm{I}]$ consists of $N$ channels, instead of infinite channels indexed by $A\in SO(2)$.  Similarly, (\ref{phi2}) is discretized as: $\forall P\in \Omega$ and $i=0,1,\cdots,N-1$,
\begin{align*}	
	&\widetilde\Phi [\bm{F}](P,i)=\frac{\nu(SO(2))}{N}\sum_{j=0}^{N-1} \widetilde\chi^{(Z_i)}_{Z_j}\,\,[\bm{F}](P,i\textcircled{+}j),
\end{align*}
where the intermediate feature map $\bm{F}$ is an $N$-channel discrete function sampled from the smooth function $so\in C^{\infty}(SO(3))$, i.e., $\bm{F}(P,i)=so(P,A_i)$, and $\textcircled{+}$ denotes the module-$N$ addition. As a result, $\widetilde \Psi$ and $\widetilde \Phi$ become discretized PDO-e{$\text{S}^\text{2}$}Convs. Particularly, batch normalization \cite{ioffe2015batch} should be implemented with a single scale and a single bias per PDO-e{$\text{S}^\text{2}$}Conv feature map in order to preserve equivariance.

\subsection{Equivariance Error Analysis}
As shown in Theorem \ref{theorem1}, the equivariance of PDO-e{$\text{S}^\text{2}$}Convs $\Psi$ and $\Phi$ is exact in the continuous domain, and it becomes approximate because of discretization in implementation. In (\ref{fp}), it is easy to verify that $O(\rho_1)=O(\rho_2)=\cdots =O(\rho_m)$ from the definition of icosahedral spherical mesh, and we write $O(\rho_i)=O(\rho)$ for simplicity. Then, we have the following equivariance error analysis.
\begin{theorem}
	$\forall \widetilde R\in SO(3),$
	\begin{align}
		&\widetilde\Psi\left[\pi_{\widetilde R}^S[\bm{I}]\right]=\pi_{\widetilde R}^{SO}\left[\widetilde\Psi [\bm{I}]\right]+O(\rho),\label{app1}\\
		&\widetilde\Phi\left[\pi_{\widetilde R}^{SO}[\bm{F}]\right]=\pi_{\widetilde R}^{SO}\left[\widetilde\Phi [\bm{F}]\right]+O(\rho)+O\left(\frac{1}{N^2}\right),\label{app2}
	\end{align}
	where transformations acting on discrete inputs and feature maps are defined as $\pi_{\widetilde R}^S[\bm{I}](P)=\pi_{\widetilde R}^S[s](P)$ and $\pi_{\widetilde R}^{SO}[\bm{F}](P,i)=\pi_{\widetilde R}^{SO}[so](P,A_i)$, respectively.
	\label{theorem4}
\end{theorem}

Particularly, we note that \cite{shen2020pdo} use PDOs to design an equivariant CNN over the Euclidean group, and achieve a quadratic order equivariance approximation for 2D images in the discrete domain. However, they can only deal with the data in the Euclidean space. Virtually, we extend their theory to the non-Euclidean geometry, i.e., the sphere. By contrast, we can only achieve a first order equivariance approximation w.r.t. the grid size $\rho$, as the representation of the sphere we use is non-Euclidean structured.

\section{Experiments \label{section6}}
We evaluate our PDO-e{$\text{S}^\text{2}$}CNNs on three datasets. The data preprocessing, model architectures and training details for each task are provided in the Supplementary Material for reproducing our results.

\subsection{Spherical MNIST Classification}
We follow \cite{cohen2018spherical} in the preparation of the spherical MNIST, and prepare non-rotated training and testing (N/N), non-rotated training and rotated testing (N/R) and rotated training and testing (R/R) tasks. The training set and the test set include 60,000 and 10,000 images, respectively. We randomly select 6,000 training images as a validation set, and choose the model with the lowest validation error during training. Inputs are on a level-4 icosahedral spherical mesh. For fair comparison with existing methods, we evaluate our method using a small and a large model, respectively.
\begin{table*}[t]
	\centering
		\begin{tabular}{l|c|ccc|c}
			\hline
			Model     &  R.E. & N/N     & N/R & R/R & \#Params\\
			\hline	
			S2CNN \cite{cohen2018spherical} & \cmark & $96$  & $94$ & $95$ & 58k   \\
			UGSCNN \cite{jiang2019spherical} &\xmark & $99.23$ & $35.60$ & $94.92$ & 62k \\
			HexRUNet-C \cite{zhang2019orientation-aware} & \xmark &  $99.45$ & $29.84$ & $97.05$ & 75k \\
			\hline
			\textbf{PDO-e{$\text{S}^\text{2}$}CNN}  & \cmark &$99.44\pm 0.06$ &$90.14\pm 0.58$ &$98.93\pm 0.08$ & 73k \\			
			\hline\hline
			SphereNet \cite{coors2018spherenet:} & \xmark & $94.4$ &- & - & 196k\\
			FFS2CNN \cite{kondor2018clebsch-gordan} & \cmark & $96.4$ & $\bm{96}$ & $96.6$ & 286k\\
			Icosahedral CNN \cite{cohen2019gauge}&\cmark & $99.43$ & $69.99$ & $99.31$ & 182k\\
			\hline
			\textbf{PDO-e{$\text{S}^\text{2}$}CNN}  & \cmark &$\bm{99.60\pm 0.04}$ &$94.25\pm 0.29$ & $\bm{99.45\pm 0.05}$ & 180k\\
			\hline
		\end{tabular}
	\caption{Results on the spherical MNIST dataset with non-rotated (N) and rotated (R) training and test data. The second column marks whether these models are rotation-equivariant (R.E.) in the spherical domain.}
	\label{tab1}
\end{table*}

As shown in Table \ref{tab1}, when using the small model (73k), our method achieves $99.44\%$ test accuracy on the N/N task. The result decreases to $90.14\%$ on the N/R task, mainly because of the equivariance error after discretization. HexRUNet-C  achieves comparable results using slightly more parameters, but it performs significantly worse on N/R and R/R tasks for lack of rotation equivariance. S2CNN performs better on the N/R task because it is nearly exactly equivariant. However, it cannot perform well on two more important tasks, N/N and R/R,  because of the distortion from nonuniform sampling. We argue that these two tasks are more important because the training and the test sets of most tasks are of identical distributions. 

When using the large model (180k), our method results in new SOTA results on the N/N and R/R tasks (99.60\% and 99.45\%), respectively, which improve the previous SOTA results (99.45\% and 99.31\%) significantly. Note that the previous SOTA results have been very competitive even for planar MNIST, and the error rates are further reduced by more than 20\% using our method. 

Also, we obtain a more competitive result (94.25\%) on the N/R task. By contrast, Icosahedral CNN only achieves $69.99\%$ test accuracy because it is only equivariant over the icosahedral group, which merely contains $60$ rotational symmetries. FFS2CNN performs the best on this task because it is also nearly exactly equivariant and use much more parameters, but it performs significantly worse on other tasks (N/N and R/R) because of the distortion in representation from nonuniform sampling.

\subsection{Omnidirectional Image Segmentation}
Omnidirectional semantic segmentation is an orientation-aware task since the natural scene images are always up-right due to gravity. We evaluate our method on the Stanford 2D-3D-S dataset \cite{2017arXiv170201105A}, which contains 1,413 equirectangular images with RGB+depth channels, and semantic labels across $13$ different classes. The input and output spherical signals are at the level-5 resolution. We use the official 3-fold cross validation to train and evaluate our model, and report the mean intersection over union (mIoU) and pixel accuracy (mAcc).

\begin{table}[t]
	\centering
	\small
	 \begin{tabular}{l|cc|c}
			\hline
			Model   & mAcc & mIoU & \#Params \\
			\hline	
			UNet  & 50.8 & 35.9 & - \\	
			Icosahedral CNN& 55.9 & 39.4 &- \\
			\hline
			\cite{eder2020tangent} & 50.9 & 38.3 & -\\
			UGSCNN & 54.7 & 38.3 & 5.18M\\
			HexRUNet & 58.6 & 43.3 & 1.59M \\
			\hline
			\textbf{PDO-e{$\text{S}^\text{2}$}CNN}  & $\bm{60.4\pm 1.0}$ & $\bm{44.6\pm 0.4}$ &  0.86M\\
			\hline
		\end{tabular}
	\caption{mAcc and mIoU comparison on 2D-3D-S at the level-5 resolution.}
	\label{tab3}
\end{table}

We report our main result in Table \ref{tab3}. As pointed out in \cite{zhang2019orientation-aware}, the 2D-3D-S dataset is acquired with preferred orientation, thus an orientation-aware system can be beneficial. Our model significantly outperforms icosahedral CNN, mainly because that our model is orientation-aware, while the latter assumes no preferred orientation. Compared with HexRUNet, an orientation-aware model, our method still performs significantly better, because we can process spherical data inherently, whereas HexRUNet can only process icosahedron data, which makes big difference. In addition, we use far fewer parameters (0.86M vs. 1.59M), showing great parameter efficiency from weight sharing across rotated filters. The detailed statistics of per-class for this task is shown in the Supplementary Material.

\subsection{Atomization Energy Prediction}
Finally, we apply our method to the QM7 dataset \cite{blum2009970,rupp2012fast}, where the goal is to regress over atomization energies of molecules given atomic positions $p_i$, and charges $z_i$. This dataset contains 7,165 molecules, and each molecule contains up to $23$ atoms of $5$ types (H, C, N, O, S). We use the official 5-fold cross validation to train and evaluate our model, and report the root mean square error (RMSE).

\begin{table}[t]
	\centering
		\begin{tabular}{l|c|c}
			\hline
			Model   &RMSE & \#Params \\
			\hline	
			MLP/Random CM  & $5.96\pm 0.48$ & -\\	
			S2CNN  & 8.47 & 1.4M\\
			FFS2CNN & 7.97 & 1.1M \\
			\hline
			\textbf{PDO-e{$\text{S}^\text{2}$}CNN}  &$\bm{3.78\pm 0.07}$ &  0.4M\\
			\hline
		\end{tabular}
	\caption{Experimental results on the QM7 task.}
	\label{tab7}
\end{table}

As shown in Table \ref{tab7}, compared with other spherical CNNs, including S2CNN and FFS2CNN, our model halves the RMSE using far fewer parameters (0.4M vs. 1M+), showing greater performance and parameter efficiency. Our method also significantly outperforms a very competitive model, the MLP trained on randomly permuted Coulomb matrices (CM) \cite{montavon2012learning}. In addition, this MLP method is unlikely to scale to large molecules, as it needs a large sample of random permutations, which grows exponentially with the numbers of molecules.

\section{Conclusions}
In this work, we define chart-based PDOs and then use them to design rotation-equivariant spherical CNNs, PDO-e{$\text{S}^\text{2}$}CNNs. PDO-e{$\text{S}^\text{2}$}CNNs are easy to implement on non-Euclidean structured representations, and we analyze the equivariance error from discretization. Extensive experiments verify the effectiveness of our method.

One drawback of our work is that the equivariance cannot be preserved as well as S2CNN and FFS2CNN do in the discrete domain. In future work, we will explore more representations of the sphere and better numerical calculation methods to improve the equivariance in the discrete domain.

\section*{ Acknowledgements}
This work was supported by the National Key Research and Development Program of China under grant 2018AAA0100205. Z. Lin is supported by NSF China (grant no.s 61625301 and 61731018), Major Scientific Research Project of Zhejiang Lab (grant no.s 2019KB0AC01 and 2019KB0AB02), Beijing Academy of Artificial Intelligence, and Qualcomm.

\bibliography{aaai_2021}
\clearpage
\section{Proof}

\begin{lemma}
	If $s \in C^{\infty}(S^2)$, $\forall R,\widetilde{R}\in SO(3)$ and $i=1,2$, we have
	\begin{equation}
		\frac{\partial}{\partial x_i^{(A_R)}} \left[\pi^{S}_{\widetilde R}[s]\right](P_R)=\frac{\partial}{\partial x_i^{(A_{\widetilde{R}^{-1}R})}} [s](P_{\widetilde R^{-1}R}),\label{first}\\
	\end{equation}
	where $(P_R,A_R)$ is the representation of $R$.
	\label{lemma1}
\end{lemma}
\begin{proof}
	Firstly, we show that $\forall R\in SO(3)$ and $i=1,2$,
	\begin{align}
		&\frac{\partial}{\partial x_i^{(A_R)}} [s](P_R)\\
		=&e_i^T\nabla_x^{(A_R)}[s]\left(P_R\right)\notag\\
		=&e_i^TA_R^{-1}\nabla_x \left[\bar{s}\cdot \varphi^{-1}_{P_R}\right](0)\notag\\
		=&e_i^TA_R^{-1}\nabla_x \left[\bar{s}\left(\bar P_R\left(
		\begin{array}{c}
			x_1\\
			x_2\\
			\sqrt{1-|x|^2}\\
		\end{array}
		\right)\right)\right]\Bigg |_{x_1=x_2=0},
		\label{11}
	\end{align}
	where $e_1=(1,0)^T$ and $e_2=(0,1)^T$. We denote $y=F(x_1,x_2)=(x_1,x_2,\sqrt{1-|x|^2})^T$, then 
	\begin{align*}
		&\frac{\partial}{\partial x_i^{(A_R)}} [s](P_R)\\
		=&e_i^TA_R^{-1}J_F(x_1,x_2)^T \bar P_R^T\nabla [\bar{s}]\left(\bar P_Ry\right)\bigg|_{x_1=x_2=0},
	\end{align*}
	where the Jacobian
	\begin{align*}
		J_F(x_1,x_2)^T=&\left(
		\begin{array}{ccc}
			\partial y_1/\partial x_1 & \partial y_2/\partial x_1 & \partial y_3/\partial x_1\\
			\partial y_1/\partial x_2 & \partial y_2/\partial x_2 & \partial y_3/\partial x_2\\
		\end{array}
		\right)\\
		=&\left(
		\begin{array}{ccc}
			1 & 0 & \frac{-x_1}{\sqrt{1-|x|^2}}\\
			0 & 1 & \frac{-x_2}{\sqrt{1-|x|^2}}\\
		\end{array}
		\right).
	\end{align*}
	So
	\begin{align}
		&\frac{\partial}{\partial x_i^{(A_R)}} [s](P_R)\notag\\
		=& e_i^TA_R^{-1}\left(I, \frac{-x}{\sqrt{1-|x|^2}} \right) \bar P_R^{-1} \nabla [\bar{s}]\left(\bar P_Ry\right)\bigg|_{x_1=x_2=0}\notag\\
		=&e_i^T\left(I, \frac{-A_R^{-1}x}{\sqrt{1-|x|^2}} \right) Z(\gamma_R)^{-1}\bar P_R^{-1} \nabla [\bar s]\left( \bar P_Ry\right)\bigg|_{x_1=x_2=0}\notag\\
		=&e_i^T\left(I, \frac{-A_R^{-1}x}{\sqrt{1-|x|^2}} \right) R^{-1}\nabla [\bar s]\left(\bar P_Ry\right)\bigg|_{x_1=x_2=0}\notag\\
		=&\left(e_i^T, 0 \right) R^{-1}\nabla [\bar{s}]\left(P_R\right).
		\label{first_left}
	\end{align}
	Thus for the right hand side of (\ref{first}),
	\begin{align*}
		\frac{\partial}{\partial x_i^{(A_{\widetilde{R}^{-1}R})}} [s](P_{\widetilde R^{-1}R})=&\left(e_i^T, 0 \right) R^{-1}\widetilde R\nabla [\bar{s}]\left(P_{\widetilde R^{-1} R}\right)\\
		=&\left(e_i^T, 0 \right) R^{-1}\widetilde R\nabla [\bar{s}]\left(\widetilde{R}^{-1}P_R\right).
	\end{align*}
	For the left hand side of (\ref{first}), we denote a spherical function $t(P)=\pi^{S}_{\widetilde R}[s](P)=s(\widetilde R^{-1}P)$, then we have
	\begin{align*}
		\frac{\partial}{\partial x_i^{(A_R)}} \left[\pi^{S}_{\widetilde R}[s]\right](P_R)=&\frac{\partial}{\partial x_i^{(A_R)}} \left[t\right](P_R)\\
		=&\left(e_i^T, 0 \right) R^{-1}\nabla \left[\bar{t}\right]\left(P_R\right).
	\end{align*}
	Obviously, we can take the extended function on the Euclidean space $\bar t(x)=\bar s(\widetilde R^{-1}x),\forall x\in \mathbb{R}^3$, then 
	\begin{equation*}
		\nabla[\bar t](P_R)=\nabla \left[\bar s\left(\widetilde{R}^{-1}x\right)\right]\Big|_{x=P_R}=\widetilde{R}\nabla \left[\bar s\right]\left(\widetilde{R}^{-1}P_R\right).
	\end{equation*}
	As a result, we have
	\begin{align*}
		\frac{\partial}{\partial x_i^{(A_R)}} \left[\pi^{S}_{\widetilde R}[s]\right](P_R)=&\left( e_i^T, 0 \right) R^{-1}\widetilde R\nabla [\bar{s}]\left(\widetilde{R}^{-1}P_R\right)\\
		=&\frac{\partial}{\partial x_i^{(A_{\widetilde{R}^{-1}R})}} [s](P_{\widetilde R^{-1}R}).
	\end{align*}
	$\hfill\blacksquare$  
	\label{proof1}
\end{proof}
\begin{lemma}
	If $s \in C^{\infty}(S^2)$, $\forall R,\widetilde{R}\in SO(3),i,j=1,2$, we have
	\begin{align}
		&\frac{\partial}{\partial x_i^{(A_R)}}\frac{\partial}{\partial x_j^{(A_R)}} \left[\pi^{S}_{\widetilde R}[s]\right](P_R)\notag\\
		=&\frac{\partial}{\partial x_i^{(A_{\widetilde{R}^{-1}R})}}\frac{\partial}{\partial x_i^{(A_{\widetilde{R}^{-1}R})}} [s](P_{\widetilde R^{-1}R}),\label{second}
	\end{align}
	where $(P_R,A_R)$ is the representation of $R$.
	\label{lemma2}
\end{lemma}
\begin{proof}
	Firstly, by definition, $\forall R,\widetilde{R}\in SO(3),i,j=1,2$,
	\begin{small}
		\begin{align*}
			&\frac{\partial}{\partial x_i^{(A_R)}}\frac{\partial}{\partial x_j^{(A_R)}} [s](P_R)\notag\\
			=&\frac{\partial}{\partial x_i^{(A_R)}}\frac{\partial}{\partial x_j^{(A_R)}} [\bar s\cdot \varphi^{-1}_{P_R}](0)\notag\\
			=&\frac{\partial}{\partial x_i^{(A_R)}}\left[e_j^T A_R^{-1}\nabla_x [\bar s\cdot \varphi^{-1}_{P_R}]\right](0)\notag\\
			=&e_i^TA_R^{-1}\nabla_x\left[e_j^T\left(I, \frac{-A_R^{-1}x}{\sqrt{1-|x|^2}} \right) R^{-1}\nabla [\bar s]\left(\bar P_Ry\right)\right]\Bigg|_{x_1=x_2=0},
		\end{align*}
	\end{small}
	where $e_1=(1,0)^T,e_2=(0,1)^T$ and $y=F(x_1,x_2)=(x_1,x_2,\sqrt{1-|x|^2})^T$. The derivation from the second line to the third line is due to (\ref{11}) and (\ref{first_left}). For ease of presentation, we denote that 
	\begin{equation*}
		h(x)=e_j^T\left( I, \frac{-A_R^{-1}x}{\sqrt{1-|x|^2}} \right) R^{-1}\nabla [\bar s]\left(\bar P_Ry\right),
	\end{equation*}
	and $h(x)=f(x)^Tg(x)$, where
	\begin{equation}
		f(x)=\left( I, \frac{-A_R^{-1}x}{\sqrt{1-|x|^2}} \right)^Te_j \label{f}
	\end{equation}
	and
	\begin{equation}
		g(x)=R^{-1}\nabla [\bar s]\left(\bar P_Ry\right).\label{g}
	\end{equation}
	As a result,
	\begin{scriptsize}
		\begin{align}
			&\frac{\partial}{\partial x_i^{(A_R)}}\frac{\partial}{\partial x_j^{(A_R)}} [s](P_R)\notag\\
			=&e_i^TA_R^{-1}\nabla_x\left[h(x)\right]\bigg|_{x_1=x_2=0}\notag\\
			=&e_i^TA_R^{-1}\nabla_x\left[f(x)^Tg(x)\right]\bigg|_{x_1=x_2=0}\notag\\
			=&e_i^TA_R^{-1}\left(Df(x)^Tg(x)\Big|_{x_1=x_2=0}+Dg(x)^Tf(x)\Big|_{x_1=x_2=0}\right)\notag\\
			=&e_i^TA_R^{-1}Df(x)^Tg(x)\bigg|_{x_1=x_2=0}+e_i^TA_R^{-1}Dg(x)^Tf(x)\bigg|_{x_1=x_2=0}.
			\label{two_terms}
		\end{align}
	\end{scriptsize}
	Firstly, we calculate the first term of the right hand side of (\ref{two_terms}). When $e_j=e_1$ in (\ref{f}) and (\ref{g}), we have 
	\begin{equation*}
		f(x)^T=\left(1,0,-\frac{\cos\gamma_R x_1+\sin \gamma_R x_2}{\sqrt{1-|x|^2}}\right),
	\end{equation*}
	then
	\begin{align*}
		Df(x)^T=&\left(0,0,-\frac{1}{\sqrt{1-|x|^2}}\left(
		\begin{array}{c}
			\cos \gamma_R\\
			\sin \gamma_R\\
		\end{array}
		\right)\right.\\
		&\left.-(\cos\gamma_R x_1+\sin \gamma_R x_2)
		\nabla_x\left[\frac{1}{\sqrt{1-|x|^2}}\right]\right).
	\end{align*}
	So
	
	\begin{align*}
		&e_i^TA_R^{-1}Df(x)^Tg(x)\bigg|_{x_1=x_2=0} \\
		= &e_i^T 
		\left(
		\begin{array}{ccc}
			0 & 0 & -1\\
			0 & 0 & 0\\
		\end{array}
		\right)
		R^{-1}\nabla [\bar s]\left(P_R\right).
	\end{align*}
	Similarly, we can get that when $e_j=e_2$,
	\begin{align*}
		&e_i^TA_R^{-1}Df(x)^Tg(x)\bigg|_{x_1=x_2=0} \\
		= &e_i^T 
		\left(
		\begin{array}{ccc}
			0 & 0 & 0\\
			0 & 0 & -1\\
		\end{array}
		\right)
		R^{-1}\nabla [\bar s]\left(P_R\right).
	\end{align*}
	In all,
	\begin{align*}
		&e_i^TA_R^{-1}Df(x)^Tg(x)\bigg|_{x_1=x_2=0} \\
		= &\left(0,0,-e_i^Te_j\right)
		R^{-1}\nabla [\bar s]\left(P_R\right).
	\end{align*}
	Now we calculate the second term of the right hand side of (\ref{two_terms}), we have
	\begin{align*}
		g(x)^T=&\left(\nabla [\bar s]\left(\bar P_Ry\right)\right)^TR\\
		=&\left(\partial_1 [\bar s]\left(\bar P_Ry\right),\partial_2 [\bar s]\left(\bar P_Ry\right),\partial_3 [\bar s]\left(\bar P_Ry\right)\right)R,
	\end{align*}
	where $\partial_k$ deontes the first-order PDO w.r.t. the $k$-th coordinate, so
	\begin{small}
		\begin{align*}
			e_i^TA_R^{-1}Dg(x)^T=&e_i^TA_R^{-1}\left(\nabla_x\left[\partial_1 [\bar s]\left(\bar P_Ry\right)\right]\right.,\\
			&\left.\nabla_x\left[\partial_2 [\bar s]\left(\bar P_Ry\right)\right],\nabla_x\left[\partial_3 [\bar s]\left(\bar P_Ry\right)\right]\right)R.
		\end{align*}
	\end{small}
	According to (\ref{11}) and (\ref{first_left}), we can get that
	\begin{align*}
		&e_i^TA_R^{-1}\nabla_x\left[\partial_k [\bar s]\left(\bar P_Ry\right)\right]\\
		=&e_i^T\left(\ I, \frac{-A_R^{-1}x}{\sqrt{1-|x|^2}} \right) R^{-1}\nabla \left[\partial_k [\bar s]\right]\left(\bar P_Ry\right),
	\end{align*}
	i.e.,
	\begin{small}
		\begin{equation*}
			e_i^TA_R^{-1}Dg(x)^T=e_i^T\left(\ I, \frac{-A_R^{-1}x}{\sqrt{1-|x|^2}} \right) R^{-1}\nabla^2  [\bar s]\left(\bar P_Ry\right)R.
		\end{equation*}
	\end{small}
	So
	\begin{align*}
		&e_i^TA_R^{-1}Dg(x)^Tf(x)\bigg|_{x_1=x_2=0}\\
		=&\left( e_i^T, 0 \right) R^{-1}\nabla^2  [\bar s]\left(P_R\right)R\left(e_j^T,0\right)^T.
	\end{align*}
	As a result, we have
	\begin{align*}
		&\frac{\partial}{\partial x_i^{(A_R)}}\frac{\partial}{\partial x_j^{(A_R)}} [s](P_R)\\
		=&e_i^TA_R^{-1}Df(x)^Tg(x)\Big|_{x_1=x_2=0}\\
		&+e_i^TA_R^{-1}Dg(x)^Tf(x)\Big|_{x_1=x_2=0}\notag\\
		=&\left(0,0,-e_i^Te_j\right)R^{-1}\nabla [\bar s]\left(P_R\right)\\
		&+\left( e_i^T, 0 \right) R^{-1}\nabla^2  [\bar s]\left(P_R\right)R\left(e_j^T,0\right)^T.
	\end{align*}
	Thus for the right hand of (\ref{second}),
	\begin{align*}
		&\frac{\partial}{\partial x_i^{(A_{\widetilde{R}^{-1}R})}}\frac{\partial}{\partial x_j^{(A_{\widetilde{R}^{-1}R})}} [s](P_{\widetilde R^{-1}R})\notag\\
		=&\left(0,0,-e_i^Te_j\right)R^{-1}\widetilde{R}\nabla [\bar s]\left(\widetilde R^{-1}P_R\right)\\
		&+\left( e_i^T, 0 \right) R^{-1}\widetilde{R}\nabla^2  [\bar s]\left(\widetilde{R}^{-1}P_R\right)\widetilde{R}^{-1}R\left(e_j^T,0\right)^T.
	\end{align*}
	As for the left hand of (\ref{second}), similar to Proof \ref{proof1}, we denote that the spherical function $t(P)=\pi^{S}_{\widetilde R}[s](P)=s(\widetilde R^{-1}P)$ and the extended 3D function $\bar t(x)=\bar s(\widetilde R^{-1}x),\forall x\in \mathbb{R}^3$, then
	\begin{small}
		\begin{align*}
			&\nabla[\bar t](P_R)=\nabla \left[\bar s\left(\widetilde{R}^{-1}x\right)\right]\Big|_{x=P_R}=\widetilde{R}\nabla \left[\bar s\right]\left(\widetilde{R}^{-1}P_R\right),\\
			&\nabla^2[\bar t](P_R)=\nabla^2 \left[\bar s\left(\widetilde{R}^{-1}x\right)\right]\Big|_{x=P_R}=\widetilde{R}\nabla^2 \left[\bar s\right]\left(\widetilde{R}^{-1}P_R\right)\widetilde{R}^{-1}.
		\end{align*}
	\end{small}
	So
	\begin{align*}
		&\frac{\partial}{\partial x_i^{(A_R)}}\frac{\partial}{\partial x_j^{(A_R)}} \left[\pi^{S}_{\widetilde R}[s]\right](P_R)\\
		=&\frac{\partial}{\partial x_i^{(A_R)}}\frac{\partial}{\partial x_j^{(A_R)}} \left[t\right](P_R)\notag\\
		=&\left(0,0,-e_i^Te_j\right)R^{-1}\nabla [\bar t]\left(P_R\right)\\
		&+\left( e_i^T, 0 \right) R^{-1}\nabla^2  [\bar t]\left(P_R\right)R\left(e_j^T,0\right)^T\notag\\
		=&\left(0,0,-e_i^Te_j\right)R^{-1}\widetilde{R}\nabla [\bar s]\left(\widetilde R^{-1}P_R\right)\\
		&+\left( e_i^T, 0 \right) R^{-1}\widetilde{R}\nabla^2  [\bar s]\left(\widetilde{R}^{-1}P_R\right)\widetilde{R}^{-1}R\left(e_j^T,0\right)^T\notag\\
		=&\frac{\partial}{\partial x_i^{(A_{\widetilde{R}^{-1}R})}}\frac{\partial}{\partial x_j^{(A_{\widetilde{R}^{-1}R})}} [s](P_{\widetilde R^{-1}R}).
	\end{align*}
	$\hfill\blacksquare$  
\end{proof}
\begin{theorem}
	If $s \in C^{\infty}(S^2)$ and $so \in C^{\infty}(SO(3))$, $\forall \widetilde{R}\in SO(3)$, we have
	\begin{align}
		\Psi \left[\pi^{S}_{\widetilde R}[s]\right]&=\pi^{SO}_{\widetilde R}\left[\Psi [s]\right],\label{equi1}\\	
		\Phi \left[\pi^{SO}_{\widetilde R}[so]\right] &= \pi^{SO}_{\widetilde R}\left[\Phi [so]\right].\label{4}
	\end{align}
	\label{theorem1}
\end{theorem}
\begin{proof}
	According to Lemmas \ref{lemma1} and \ref{lemma2}, $\forall R,\widetilde R\in SO(3)$,
	\begin{small}
		\begin{align}
			&\Psi \left[\pi^{S}_{\widetilde R}[s]\right](R)\notag\\
			=&\chi^{(A_R)} \left[\pi^{S}_{\widetilde R}[s]\right](P_R)\notag\\
			=&\left(w_1 + w_2\frac{\partial}{\partial x_1^{(A_R)}}+ w_3\frac{\partial}{\partial x_2^{(A_R)}}+ w_4\frac{\partial}{\partial x_1^{(A_R)}}\frac{\partial}{\partial x_1^{(A_R)}}\right.\notag\\
			&\left.+w_5\frac{\partial}{\partial x_1^{(A_R)}}\frac{\partial}{\partial x_2^{(A_R)}}+w_6\frac{\partial}{\partial x_2^{(A_R)}}\frac{\partial}{\partial x_2^{(A_R)}} \right)\left[\pi^{S}_{\widetilde R}[s]\right](P_R)\notag\\
			=&\left(w_1 + w_2\frac{\partial}{\partial x_1^{(A_{\widetilde{R}^{-1}R})}}+ w_3\frac{\partial}{\partial x_2^{(A_{\widetilde{R}^{-1}R})}}\right.\notag\\
			&\left.+ w_4\frac{\partial}{\partial x_1^{(A_{\widetilde{R}^{-1}R})}}\frac{\partial}{\partial x_1^{(A_{\widetilde{R}^{-1}R})}}+w_5\frac{\partial}{\partial x_1^{(A_{\widetilde{R}^{-1}R})}}\frac{\partial}{\partial x_2^{(A_{\widetilde{R}^{-1}R})}}\right.\notag\\
			&\left.+w_6\frac{\partial}{\partial x_2^{(A_{\widetilde{R}^{-1}R})}}\frac{\partial}{\partial x_2^{(A_{\widetilde{R}^{-1}R})}} \right)[s]\left(P_{\widetilde{R}^{-1}R}\right)\notag\\
			=&\chi^{(A_{\widetilde{R}^{-1}R})} [s](P_{\widetilde{R}^{-1}R})\notag\\
			=&\pi^{SO}_{\widetilde R}\left[\Psi [s]\right](R).
			\label{detail}
		\end{align}
	\end{small}
	So (\ref{equi1}) is satisfied. As for (\ref{4}),
	\begin{align*}
		&\Phi \left[\pi^{SO}_{\widetilde{R}}[so]\right]\left(P_R,A_R\right) \\
		=&\int_{SO(2)} \chi^{(A_R)}_{A}\left[so\left(\widetilde R ^{-1}P,A_{\widetilde{R}^{-1}R}A\right)\right]\Bigg |_{P=P_R} d\nu(A)\notag\\
		=&\int_{SO(2)} \chi^{(A_R)}_{A}\left[\pi^{S}_{\widetilde R}\left[so\left(P,A_{\widetilde{R}^{-1}R}A\right)\right]\right]\Bigg |_{P=P_R} d\nu(A)\notag\\
		= &\int_{SO(2)} \chi^{(A_{\widetilde R^{-1}R})}_{A}\left[so\right]\left(P_{\widetilde R^{-1}R},A_{\widetilde{R}^{-1}R}A\right) d\nu(A)\notag\\
		= &\pi^{SO}_{\widetilde{R}}\left[\int_{SO(2)} \chi^{(A_{R})}_{A}\left[so\right]\left(P_{R},A_{R}A\right) d\nu(A)\right]\notag\\
		=& \pi^{SO}_{\widetilde{R}}\left[\Phi[so]\right]\left(P_R,A_R\right).
	\end{align*}
	The derivation from the third line to the fourth line is due to (\ref{detail}). So (\ref{4}) is satisfied.$\hfill\blacksquare$  
\end{proof}

\begin{theorem}
	If $ s \in C^{\infty}(S^2)$, $\forall \widetilde{R}\in SO(3)$, we have
	\begin{align*}
		T\left[\pi^{S}_{\widetilde R}[s]\right]= \pi^{SO}_{\widetilde R}\left[T[s]\right].
	\end{align*}
	\label{theorem3}
\end{theorem}
\begin{proof}
	According to Theorems \ref{theorem1}, we have
	\begin{align*}
		T\left[\pi^{S}_{\widetilde R}[s]\right]=&\Phi^{(L)}\left[\cdots\sigma\left(\Phi^{(1)}\left[\sigma\left(\Psi\left[\pi^{S}_{\widetilde R}[s]\right]\right)\right]\right)\right]\notag\\
		=&\Phi^{(L)}\left[\cdots\sigma\left(\Phi^{(1)}\left[\sigma\left(\pi^{SO}_{\widetilde R}\left[\Psi[s]\right]\right)\right]\right)\right]\notag\\
		=&\Phi^{(L)}\left[\cdots\sigma\left(\Phi^{(1)}\left[\pi^{SO}_{\widetilde R}\left[\sigma\left(\Psi[s]\right)\right]\right]\right)\right]\notag\\
		=&\Phi^{(L)}\left[\cdots\sigma\left(\pi^{SO}_{\widetilde R}\left[\Phi^{(1)}\left[\sigma\left(\Psi[s]\right)\right]\right]\right)\right]\notag\\
		=& \pi^{SO}_{\widetilde R}\left[\Phi^{(L)}\left[\cdots\sigma\left(\Phi^{(1)}\left[\sigma(\Psi[s])\right]\right)\right]\right]\notag\\
		=&\pi^{SO}_{\widetilde R}\left[T[s]\right].
	\end{align*}
	$\hfill\blacksquare$  
\end{proof}

\begin{lemma}
	$\forall P\in\Omega$ and $\bm{w}\in \mathcal{\mathbb{R}}^5$,
	\begin{equation*}
		\bm{w}^TD_P=\bm{w}^T\hat D_P +O(\rho).
	\end{equation*}
	\label{lemma3}
\end{lemma}
\begin{proof}
	According to (13) in the main body, we have
	\begin{equation*}
		F_P = V_P D_P + O(\rho^3),
	\end{equation*}
	and then
	\begin{align*}
		D_P = &(V_P^TV_P)^{-1}V_P D_P + (V_P^TV_P)^{-1}V_PO(\rho)\\
		=&\hat D_P + (V_P^TV_P)^{-1}V_PO(\rho^3).
	\end{align*}
	Actually, 
	
	\begin{equation*}
		V_P=\left[
		\begin{array}{ccccc}
			\vdots & \vdots  & \vdots  & \vdots  & \vdots  \\
			x_{i1} & x_{i2} & \frac{1}{2}x_{i1}^2 & x_{i1}x_{i2} & \frac{1}{2}x_{i2}^2\\
			\vdots  & \vdots  & \vdots & \vdots  & \vdots  \\
		\end{array}
		\right]
		= XC,
	\end{equation*}
	where 
	\begin{equation*}
		X = \left[
		\begin{array}{ccccc}
			\vdots & \vdots  & \vdots  & \vdots  & \vdots  \\
			\frac{x_{i1}}{\rho} & \frac{x_{i2}}{\rho} & \frac{x_{i1}^2}{2\rho^2} & \frac{x_{i1}x_{i2}}{\rho^2} & \frac{x_{i2}^2}{2\rho^2}\\
			\vdots  & \vdots  & \vdots & \vdots  & \vdots  \\
		\end{array}
		\right]
	\end{equation*}
	and
	\begin{equation*}
		C = \left[
		\begin{array}{cc}
			\rho I_2 & 0 \\
			0 & \rho^2 I_3 \\
		\end{array}
		\right].
	\end{equation*}
	Obviously $X=O(1)$, so we have
	\begin{align*}
		(V_P^TV_P)^{-1}V_PO(\rho^3) =& C^{-1}(X^TX)^{-1}X^TO(\rho^3)\\
		=&
		\left[
		\begin{array}{cc}
			\frac{I_2}{\rho} & 0 \\
			0 & \frac{I_3}{\rho^2} \\
		\end{array}
		\right]O(\rho^3)\\
		=&
		\left[
		\begin{array}{c}
			O(\rho^2)\bm{1}_2\\
			O(\rho)\bm{1}_3\\
		\end{array}
		\right],
	\end{align*}
	i.e.,
	$\forall \bm{w}\in \mathbb{R}^5$,
	\begin{align*}
		\bm{w}^TD_P=&\bm{w}^T\hat D_P +\bm{w}^T(V_P^TV_P)^{-1}V_PO(\rho^3)\\
		=&\bm{w}^T\hat D_P +O(\rho).
	\end{align*}
	$\hfill\blacksquare$  
\end{proof}

\begin{theorem}
	$\forall \widetilde R\in SO(3),$
	\begin{align}
		&\widetilde\Psi\left[\pi_{\widetilde R}^S[\bm{I}]\right]=\pi_{\widetilde R}^{SO}\left[\widetilde\Psi [\bm{I}]\right]+O(\rho),\label{app1}\\
		&\widetilde\Phi\left[\pi_{\widetilde R}^{SO}[\bm{F}]\right]=\pi_{\widetilde R}^{SO}\left[\widetilde\Phi [\bm{F}]\right]+O(\rho)+O\left(\frac{1}{N^2}\right),\label{app2}
	\end{align}
	where transformations acting on discrete inputs and feature maps are defined as $\pi_{\widetilde R}^S[\bm{I}](P)=\pi_{\widetilde R}^S[s](P)$ and $\pi_{\widetilde R}^{SO}[\bm{F}](P,i)=\pi_{\widetilde R}^{SO}[so](P,A_i)$, respectively.
	\label{theorem4}
\end{theorem}
\begin{proof}
	$\forall i= 0,1,\cdots,N-1$, the operator $\chi^{(Z_i)}$ is a linear combination of differential operators and $\widetilde \chi^{(Z_i)}$ is a linear combination of corresponding numerical estimations, except a trivial scalar. According to Lemma \ref{lemma3}, we have that $\forall P\in \Omega$,
	\begin{align*}
		\chi^{(Z_i)}[s](P) &= \widetilde\chi^{(Z_i)}[\bm{I}](P)+O(\rho),\\
		\chi^{(Z_i)}\left[\pi_{\widetilde R}^S[s]\right](P) &= \widetilde\chi^{(Z_i)}\left[\pi_{\widetilde R}^S[\bm{I}]\right](P)+O(\rho),
	\end{align*}
	i.e.,
	\begin{align}
		\Psi[s](P,Z_i) &= \widetilde\Psi[\bm{I}](P,i)+O(\rho),\notag\\
		\Psi\left[\pi_{\widetilde R}^S[s]\right](P,Z_i) &= \widetilde\Psi\left[\pi_{\widetilde R}^S[\bm{I}]\right](P,i)+O(\rho).\label{left1}
	\end{align}
	Easily, we have 
	\begin{equation}
		\pi_{\widetilde R}^{SO}\left[\Psi[s]\right](P,Z_i) = \pi_{\widetilde R}^{SO}\left[\widetilde\Psi[\bm{I}]\right](P,i)+O(\rho).\label{left2}
	\end{equation}
	From (\ref{equi1}) we know that the left hand sides of (\ref{left1}) and (\ref{left2}) equal, hence the right hand sides of the two equations are the same, which results in (\ref{app1}). 
	
	As for (\ref{app2}),
	\begin{align*}
		&\Phi [so](P,Z_i)\\
		=&\int_{SO(2)} \chi^{(Z_i)}_{Z}\,\,[so](P,Z_iZ) d\nu(Z)\\
		=&\frac{\nu(SO(2))}{N}\sum_{j=0}^{N-1} \chi^{(Z_i)}_{Z_j}[so](P,Z_iZ_j)+O\left(\frac{1}{N^2}\right)\\
		=& \frac{\nu(SO(2))}{N}\sum_{j=0}^{N-1} \left(\widetilde\chi^{(Z_i)}_{Z_j}[\bm{F}](P,i\textcircled{+}j)+O(\rho)\right) +O\left(\frac{1}{N^2}\right)\\
		=& \widetilde{\Phi}[\bm{F}](P,i)+O(\rho)+O\left(\frac{1}{N^2}\right).
	\end{align*}
	Then we can prove (\ref{app2}) analogously.$\hfill\blacksquare$  
\end{proof}

\section{Equivariant Network Architectures}

When the inputs and feature maps consist of multiple channels, we utilize multiple $\Psi$'s and $\Phi$'s to process inputs and generate outputs. To be specific, for the input layer, where inputs $s$ consist of $M_s$ channels and the resulting feature maps $so^{(1)}$ consist of $M_1$ layer, we have

\begin{equation*}
	\left[
	\begin{array}{c}
		so_1^{(1)}\\
		\vdots\\
		so_{M_1}^{(1)}
	\end{array}
	\right]
	=\sigma\left(
	\left[
	\begin{array}{ccc}
		\Psi_{11} & \cdots & \Psi_{1M_s}\\
		\vdots  & \ddots&\vdots \\
		\Psi_{M_1 1}  & \cdots & \Psi_{M_1M_s}\\
	\end{array}
	\right]
	\left[
	\begin{array}{c}
		s_1\\
		\vdots\\
		s_{M_s}\\
	\end{array}
	\right]\right).
\end{equation*}
For the following layer, where feature maps $so^{(l)}$ at the $l$-th layer consist of $M_l$ channels, we have
\begin{scriptsize}
	\begin{equation*}
		\left[
		\begin{array}{c}
			so_1^{(l+1)}\\
			\vdots\\
			so_{M_{l+1}}^{(l+1)}
		\end{array}
		\right]
		=\sigma\left(
		\left[
		\begin{array}{ccc}
			\Phi^{(l)}_{11}  & \cdots & \Phi^{(l)}_{1M_l}\\
			\vdots  & \ddots&\vdots \\
			\Phi^{(l)}_{M_{l+1} 1} & \cdots & \Phi^{(l)}_{M_{l+1}M_l}\\
		\end{array}
		\right]
		\left[
		\begin{array}{c}
			so_1^{(l)}\\
			\vdots\\
			so_{M_{l}}^{(l)}\\
		\end{array}
		\right]\right).
		\label{mphi}
	\end{equation*}
\end{scriptsize}
Finally, we obtain a more general network architecture, and it is easy to verify that equivariance can still be preserved through this network. 

Particularly, as for 
\begin{align}
	\Phi [so](P_R,A_R)=\int_{SO(2)} \chi^{(A_R)}_{A}\,\,[so](P_R,A_RA) d\nu(A)\label{phi2},
\end{align}
if we take $w_{A,i}=0$ for any $A\in SO(2)$ and $i=2,3,\cdots,6$, then $(\ref{phi2})$ can be rewritten as 
\begin{equation}
	\Phi [so](P_R,A_R)=\int_{SO(2)} w_{A,1}so(P_R,A_RA) d\nu(A),\label{one}
\end{equation}
which is analogous to the conventional $1\times 1$ convolution in planar CNNs.

\section{Model Architectures and Training Details}
In this section we provide network architectures and training details for reproducing our results in experiments. Each experiment is run for $5$ times and implemented using Pytorch.
\subsection{Spherical MNIST Classification}
The small model consists of 4 convolution layers and 3 fully connected (FC) layers. The convolution layers have $8, 12, 16$ and $28$ output channels, and the FC layers have $28,28$ and $10$ channels, respectively. The large model consists of 5 convolution layers and 3 fully connected (FC) layers. The convolution layers have $8, 12, 16, 24,$ and $48$ output channels, and the FC layers have $48,48$ and $10$ channels, respectively. $N$ is set to $16$, and downsampling is performed after layer 2. In between convolution layers, we use batch normalization \cite{ioffe2015batch} and ReLU nonlinearities. 

The models are trained using the Adam algorithm \cite{kingma2014adam}. We use generalized He's weight initialization scheme introduced in \cite{weiler2018learning} for the convolution layers and Xavier initialization \cite{glorot2010understanding} for the FC layers. For N/R task, we use dropout for better generalization. We train using a batch size of $16$ for $80$ epochs, an initial learning rate of $0.01$ and a step decay of $0.5$ per $10$ epochs. We use the cross-entropy loss for training the classification network.

\begin{table*}[t]

	\centering
	\scalebox{0.85}{
		\linespread{1.0}\selectfont
		\begin{tabular}{l|c|ccccccccccccc}
			\hline
			Model   & Mean & beam & board & bookcase & ceiling & chair & clutter &column & door & floor & sofa & table & wall & window  \\
			\hline	
			UNet  & 50.8 & 17.8 & 40.4 & 59.1 & 91.8 & 50.9 & 46.0 & 8.7 & 44.0 & 94.8 & 26.2 & 68.6 & 77.2 & 34.8 \\	
			UGSCNN & 54.7 & 19.6 & 48.6 & 49.6 & 93.6 & 63.8 & 43.1 & 28.0 & 63.2 & 96.4 & 21.0 & 70.0 & 74.6 & 39.0 \\
			Icosahedral CNN & 55.9 & - & -& -& -& -& -& -& -& -& -& -& -& -\\
			HexRUNet & 58.6 & \textbf{23.2} & 56.5 & \textbf{62.1} & \textbf{94.6} & 66.7 & 41.5 & 18.3 & \textbf{64.5} & 96.2& \textbf{41.1} & \textbf{79.7} & 77.2 & 41.1 \\
			\hline
			PDO-e{$\text{S}^\text{2}$}CNN  & \textbf{60.4} & 22.2 & \textbf{59.6} & 59.7 & 93.5& \textbf{67.4} & \textbf{53.9} & \textbf{26.3} & 64.1&\textbf{97.1}&30.8&75.4&\textbf{81.9}& \textbf{53.4}\\
			\hline
		\end{tabular}
	}
	\caption{mAcc comparison on 2D-3D-S dataset. Per-class accuracy is shown when available.} 
	\label{tab4}
\end{table*}

\begin{table*}[t]
	\centering
	\scalebox{0.85}{
		\linespread{1.0}\selectfont
		\begin{tabular}{l|c|ccccccccccccc}
			\hline
			Model   & Mean & beam & board & bookcase & ceiling & chair & clutter &column & door & floor & sofa & table & wall & window  \\
			\hline	
			UNet  & 35.9 & 8.5 & 27.2 & 30.7 & 78.6 & 35.3 & 28.8 & 4.9 & 33.8 & 89.1 & 8.2 & 38.5 & 58.8 & 23.9 \\	
			UGSCNN & 38.3 & 8.7 & 32.7 & 33.4 & 82.2 & 42.0 & 25.6 & 10.1 & 41.6 & 87.0 & 7.6 & 41.7 & 61.7 & 23.5 \\
			Icosahedral CNN& 39.4 & - & -& -& -& -& -& -& -& -& -& -& -& -\\
			HexRUNet & 43.3 & 10.9 & 39.7 & 37.2 & \textbf{84.8} & \textbf{50.5} & 29.2 & 11.5 & 45.3 & \textbf{92.9} & \textbf{19.1} & 49.1 & 63.8 & 29.4 \\
			\hline
			PDO-e{$\text{S}^\text{2}$}CNN  & \textbf{44.6} &\textbf{11.4} &\textbf{43.3} &\textbf{38.2} &83.9& 50.3&\textbf{31.3}&\textbf{12.4}& \textbf{48.4}& 90.0&18.1&\textbf{49.5}&\textbf{65.9} & \textbf{37.1}\\
			\hline
		\end{tabular}
	}
	\caption{mIoU comparison on 2D-3D-S dataset. Per-class IoU is shown when available.}
	\label{tab5}
\end{table*}

\subsection{Omnidirectional Image Segmentation}
Following \cite{jiang2019spherical}, we preprocess the data into  a spherical signal by sampling the original rectangular images at the latitude-longitudes of the spherical mesh vertex positions. The input RGB-D channels are interpolated using bilinear interpolation, and semantic labels are acquired using nearest-neighbor interpolation. The input and output spherical signals are at the level-5 resolution. 

The network architecture is a residual U-Net \cite{he2016deep,ronneberger2015u} using PDO-e{$\text{S}^\text{2}$}Convs, which consists of an encoder and a decoder. The encoder network takes as input a signal at resolution $r=5$. We use a similar network architecture as that in \cite{jiang2019spherical} and the details are shown in Table \ref{tab2}, and $N$ is set to $8$ except the last layer. We use a trivial PDO-e{$\text{S}^\text{2}$}Conv ($N=1$) for the last layer to obtain 15 output channels. Note that we use $15$ output channels since the 2D-3D-S dataset has two additional classes (invalid and unknown) that are not evaluated for performance.

\begin{table}[h]
	\centering
	\linespread{1.0}\selectfont
	\resizebox{0.9\columnwidth}{!}{
		\smallskip \begin{tabular}{ccccccc}
			\hline
			Level & a & Block & b & c & s & $N$\\
			\hline	\hline
			5  & 4 & Encoder & -  & 16 & 2 &8\\	
			4  & 16 & Encoder & 16 & 32 & 2 &8\\
			3  & 32 & Encoder & 32 & 64 & 2 & 8\\
			\hline \hline
			2 & 64 & Decoder & - & 64 &  0.5 & 8\\
			3 & 64 & Decoder & - & 64 &  1 & 8\\
			3 & $64\times 2$ & Decoder & 32 & 32 &  0.5 & 8\\
			4 & 32 & Decoder & - & 32 &  1 & 8\\
			4 & $32\times 2$ & Decoder & 16 & 16 &  0.5 & 8\\
			5 & 16 & Decoder & - & 16 &  1 & 8\\
			5 & $16\times 2$ & Decoder & 16 & 16 &  1 & 8\\
			\hline\hline
			5 & 16 & PDO-e{$\text{S}^\text{2}$}Conv & - & 8 &  1 & 8\\
			5 & $8\times 8$ & PDO-e{$\text{S}^\text{2}$}Conv & -  & 15  &  1 & 1\\
			\hline
		\end{tabular}
	}
	\caption{The architecture of PDO-e{$\text{S}^\text{2}$}CNN used in the 2D-3D-S image segmentation experiments. a, b, c and s stands for input channels, bottleneck channels, output channels, and strides, respectively. When $s=2$, downsampling is performed using average pooling, and when $s=0.5$, upsampling is applied using linear interpolation.}
	\label{tab2}
\end{table}

We use the Adam optimizer to train our network for $200$ epochs, with an initial learning rate of $0.01$ and a step decay of $0.9$ per $20$ epochs. Following \cite{jiang2019spherical}, we use the weighted cross-entropy loss for training, and the loss for each class is of the following weighting scheme:
\begin{equation*}
	w_c=\frac{1}{1.02+\log(f_c)},
\end{equation*}
where $w_c$ is the weight corresponding to class $c$, and $f_c$ is the frequency by which class $c$ appears in the training set. We use zero weight for the two dropped classes (invalid and unknown). The detailed statistics for this task is shown in Tables \ref{tab4} and \ref{tab5}.

\subsection{Atomization Energy Prediction}
Following \cite{cohen2018spherical}, we represent each molecule as a spherical signal. Specifically, we define a sphere $S_i$ around $p_i$ for each atom $i$. The radius is kept uniform across atoms and molecules, and chosen minimal such that no intersections among spheres happen. We define potential functions $U_z = \sum_{j\neq i,z_j=z}\frac{z_iz}{|x-p_i|}$ and produce a $T$ channel spherical signal for each atom in the molecule. Finally, we represent these signals on a level-3 mesh. 

The architecture used on QM7 dataset is shown in Table \ref{tab6} and $N$ is set to $8$. We share weights among atoms making filters permutation invariant, by pushing the atom dimension into the batch dimension. We use global spatial pooling and orientation pooling after the last PDO-e{$\text{S}^\text{2}$}Conv. Next, we use DeepSet \cite{zaheer2017deep} to refine the resulting feature vectors. Both PDO-e{$\text{S}^\text{2}$}CNN and DeepSet are jointly optimized. Following \cite{cohen2018spherical}, we firstly train a simple MLP only on the $5$ frequencies of atom types in a molecule, and then train our main model on the residual. Specifically, we use the Adam optimizer to train this model using a batch size of $32$ for $30$ epochs, an initial learning rate of $0.001$ and a step decay of $0.1$ per $10$ epochs.

\begin{table}[h]
	\centering
	\linespread{1.0}\selectfont
	\resizebox{0.9\columnwidth}{!}{
		\smallskip \begin{tabular}{llcc}
			\hline
			PDO-e{$\text{S}^\text{2}$}CNN & Layer & Channels & Level\\
			\hline
			& PDO-e{$\text{S}^\text{2}$}Conv & 16 & 3\\	
			& PDO-e{$\text{S}^\text{2}$}Conv & 32 & 2\\
			& PDO-e{$\text{S}^\text{2}$}Conv & 64 & 1\\
			& PDO-e{$\text{S}^\text{2}$}Conv & 64 & 0\\
			& Global orientation pooling \\
			& Global spatial pooling \\
			\hline 
			DeepSet & Layer & Input/Hidden\\
			\hline
			& $\phi (\text{MLP})$ & 64/256\\
			& $\psi (\text{MLP})$ & 64/512\\
			\hline
		\end{tabular}
	}
	\label{tab6}
		\caption{The architecture used in QM7 atomization energy prediction experiments. Downsampling is performed using average pooling.}
\end{table}
\end{document}